\documentclass[11pt]{article}
\pdfoutput=1

\usepackage{lastpage}
\usepackage[utf8]{inputenc} 
\usepackage[T1]{fontenc}

\usepackage{geometry}
\usepackage{amssymb,amsmath,thmtools,amsfonts,amsthm}
\usepackage{graphicx}
\usepackage{url}
\usepackage{setspace}
\usepackage[pdftex,bookmarksnumbered,bookmarksopen,
colorlinks,citecolor=black,linkcolor=black,urlcolor=black]{hyperref}
\usepackage{xcolor}
\usepackage{todonotes}
\usepackage{microtype}
\usepackage{graphicx}
\usepackage{subfigure}
\usepackage{booktabs} 

\usepackage{colortbl}

\definecolor{maroon}{cmyk}{0,0.87,0.68,0.32}

\usepackage{times}
\usepackage{cleveref}
\usepackage{graphicx}
\usepackage{natbib}

\usepackage{booktabs}
\usepackage{url}

\newcommand \reals {\mathbb{R}}
\newcommand \integers {\mathbb{Z}}
\newcommand \expect {\operatorname*{\mathbb{E}}}
\newcommand \prob {\operatorname*{Pr}}
\newcommand \norm [1]{\left\Vert#1\right\Vert}
\newcommand \argmax {\operatorname*{argmax}}

\newcommand \tv {\operatorname{TV}}
\newcommand \ball [1]{\mathcal{B}_{#1}}
\newcommand \tvball {\ball{\text{TV}}}
\newcommand \normal {\mathcal{N}}
\newcommand \uniform {\mathcal{U}}
\newcommand \ind {\operatorname{\mathbb{I}}}
\newcommand \vol {\operatorname{Vol}}

\newcommand \cD {\mathcal{D}}
\newcommand \cA {\mathcal{A}}
\newcommand \cY {\mathcal{Y}}
\newcommand \cS {\mathcal{S}}

\newcommand \bQ {\mathbf{Q}}

\newcommand{\comment}[1]{}

\newtheorem{theorem}{Theorem}[section]
\newtheorem{lemma}[theorem]{Lemma}

\newtheorem{corollary}[theorem]{Corollary}

\newtheorem{definition}{Definition}

\title{Random Smoothing Might be Unable to Certify $\ell_\infty$ Robustness for High-Dimensional Images}

\author{
Avrim Blum \\ TTIC  \\ \small{avrim@ttic.edu} \and
Travis Dick \\ University of Pennsylvania \\ \small{tbd@seas.upenn.edu} \and
Naren Manoj \\ TTIC \\
\small{nsm@ttic.edu} \and
Hongyang Zhang\thanks{Corresponding author.} \\ TTIC \\
\small{hongyanz@ttic.edu}
}

\date{}

\begin{document}
\maketitle

\begin{abstract}
We show a hardness result for random smoothing to achieve certified
  adversarial robustness against attacks in the $\ell_p$ ball of radius
  $\epsilon$ when $p>2$. Although random smoothing has been well understood for
  the $\ell_2$ case using the Gaussian distribution, much remains unknown
  concerning the existence of a noise distribution that works for the case of
  $p>2$. This has been posed as an open problem by \citet{cohen2019certified}
  and includes many significant paradigms such as the $\ell_\infty$ threat
  model. In this work, we show that \emph{any}
  noise distribution $\cD$ over $\reals^d$ that provides $\ell_p$ robustness for all base classifiers with $p>2$ must satisfy $\expect
  \eta_i^2=\Omega(d^{1-2/p}\epsilon^2(1-\delta)/\delta^2)$ for 99\% of the features
  (pixels) of vector $\eta\sim\cD$, where $\epsilon$ is the robust radius and $\delta$ is the score gap between the highest-scored class and the runner-up. Therefore, for high-dimensional images
  with pixel values bounded in $[0,255]$, the required noise will eventually
  dominate the useful information in the images, leading to trivial smoothed
  classifiers.

\end{abstract}

\section{Introduction} \label{sec:introduction}

Adversarial robustness has been a critical object of study in various fields,
including machine learning~\citep{zhang2019theoretically,madry2017towards},
computer vision~\citep{szegedy2013intriguing,yang2019design}, and many other
domains~\citep{lecuyer2019certified}. In machine learning and computer vision,
the study of adversarial robustness has led to significant advances in defending
against attacks in the form of perturbed input images, where the data is high
dimensional but each feature is bounded in $[0,255]$. The problem can be stated
as that of learning a non-trivial classifier with high test accuracy on the
adversarial images. The adversarial perturbation is either restricted to be in
an $\ell_p$ ball of radius $\epsilon$ centered at $0$, or is measured under
other threat models such as Wasserstein distance and adversarial
rotation~\citep{wong2019wasserstein,brown2018unrestricted}. The focus of this
work is the former setting.

Despite a large amount of work on adversarial robustness, many fundamental
problems remain open. One of the challenges is to end the long-standing arms
race between adversarial defenders and attackers: defenders design empirically
robust algorithms which are later exploited by new attacks designed to undermine
those defenses~\citep{athalye2018obfuscated}. This motivate the  study of
\emph{certified robustness}~\citep{raghunathan2018semidefinite,wong2018scaling}---algorithms that are provably robust to the
worst-case attacks---among which random smoothing~\citep{cohen2019certified,li2019certified,lecuyer2019certified} has received significant
attention in recent years. Algorithmically, random smoothing takes a base
classifier $f$ as an input, and outputs a smooth classifier $g$ by repeatedly
adding i.i.d. noises to the input examples and outputting the most-likely class.
Random smoothing has many appealing properties that one could exploit: it is
agnostic to network architecture, is scalable to deep networks, and perhaps most
importantly, achieves state-of-the-art certified $\ell_2$ robustness for deep
learning based
classifiers~\citep{cohen2019certified,li2019certified,lecuyer2019certified}.

\medskip
\noindent{\textbf{Open problems in random smoothing.}} Given the rotation
invariance of Gaussian distribution, most positive results for random smoothing
have focused on the $\ell_2$ robustness achieved by smoothing with the Gaussian
distribution (see Theorem \ref{thm: cohen's result}). However, the existence of
a noise distribution for general $\ell_p$ robustness has been posed as an open
question by \citet{cohen2019certified}:
\begin{center}
  \emph{We suspect that smoothing with other noise distributions may lead to
  similarly natural robustness guarantees for other perturbation sets such as
  general $\ell_p$ norm balls.}
\end{center}
Several special cases of the conjecture have been proven for $p < 2$.
\citet{li2019certified} show that $\ell_1$ robustness can be achieved with the
Laplacian distribution, and \citet{lee2019tight} show that $\ell_0$ robustness
can be achieved with a discrete distribution. Much remains unknown concerning
the case when $p>2$. On the other hand, the most standard threat model for
adversarial examples is $\ell_\infty$ robustness, among which 8-pixel and
16-pixel attacks have received significant attention in the computer vision
community (i.e., the adversary can change every pixel by 8 or 16 intensity
values, respectively).
In this paper, we derive lower bounds on the magnitude of noise required for
certifying $\ell_p$ robustness that highlights a phase transition at $p=2$. In
particular, for $p > 2$, the noise that must be added to each feature of the
input examples grows with the dimension $d$ in expectation, while it can be constant for $p
\leq 2$.

\medskip
\noindent{\textbf{Preliminaries.}} Given a base classifier
$f:\reals^d\rightarrow \cY$ and smoothing distribution $\cD$, the randomly
smoothed classifier is defined as follows: for each class $y \in \cY$, define
the score of class $y$ at point $x$ to be $G_y(x; \cD, f) = \prob_{\eta \sim
\cD}(f(x + \eta) = y)$. Then the smoothed classifier outputs the class with the
highest score: $g(x; \cD, f) = \argmax_y G_y(x; \cD, f)$.

The key property of smoothed classifiers is that the scores $G_y(x; \cD, f)$
change slowly as a function of the input point $x$ (the rate of change depends
on $\cD$). It follows that if there is a gap between the highest and second
highest class scores at a point $x$, the smoothed classifier $g(\cdot; \cD, f)$
must be constant in a neighborhood of $x$. We denote the score gap by $\Delta(x;
\cD, f) = G_a(x; \cD, f) - G_b(x; \cD, f)$, where $a = \argmax_{y} G_y(x; \cD,
f)$ and $b = \argmax_{y \neq a} G_y(x; \cD, f)$.

\begin{definition}[$(\mathcal{A},\delta)$- and $(\epsilon,\delta)$-robustness]
For any set $\mathcal{A} \subseteq \reals^d$ and $\delta \in [0,1]$, we say that the
smoothed classifier $g$ is \emph{$(\mathcal{A},\delta)$-robust} if for all $x \in
\reals^d$ with $\Delta(x; \cD, f) > \delta$, we have that $g(x+v; \cD, f) = g(x;
\cD, f)$ for all $v \in \mathcal{A}$. For a given norm $\norm{\cdot}$, we also say that
$g$ is $(\epsilon, \delta)$-robust with respect to $\norm{\cdot}$ if it is
$(\mathcal{A},\delta)$-robust with $\mathcal{A} = \{ v \in \reals^d: \norm{v} \leq \epsilon\}$.
\end{definition}

When the base classifier $f$ and the smoothing distribution $\cD$ are clear from
context, we will simply write $G_y(x)$, $g(x)$, and $\Delta(x)$. We often refer
to a sample from the distribution $\cD$ as \emph{noise}, and use \emph{noise
magnitude} to refer to squared $\ell_2$ norm of a noise sample. Finally, we use $\cD+v$
to denote the distribution of $\eta + v$, where $\eta \sim \cD$.

\subsection{Our results}

Our main results derive lower bounds on
the magnitude of noise sampled from any distribution $\cD$ that leads to
$(\epsilon, \delta)$-robustness with respect to $\norm{\cdot}_p$ for all
possible base classifiers $f : \reals^d \to \cY$.
A major strength of random smoothing is that it provides certifiable
robustness guarantees without making any assumption on the base classifier $f :
\reals^d \to \cY$. For example, the results of \citet{cohen2019certified}
imply that using a Gaussian smoothing distribution with standard deviation
$\sigma = \frac{2\epsilon}{\delta}$ guarantees that $g(\cdot; \cD, f)$ is
$(\epsilon, \delta)$-robust with respect to $\norm{\cdot}_2$ for every possible
base classifier $f : \reals^d \to \cY$. We show that there is a phase
transition at $p = 2$, and that ensuring $(\epsilon,\delta)$-robustness for all
base classifiers $f$ with respect to $\ell_p$ norms with $p > 2$ requires that
the noise magnitude grows non-trivially with the dimension $d$ of the input
space. In particular, for image classification tasks where the data is high
dimensional and each feature is bounded in the range $[0,255]$, this implies
that for sufficiently large dimensions, the necessary noise will dominate the
signal in each example.

The following result, proved in \Cref{app:TV}, shows that any distribution $\cD$
that provides $(\cA, \delta)$-robustness for every possible base classifier $f :
\reals^d \to \cY$ must be approximately translation-invariant to all
translations $v \in \cA$. More formally, for every $v \in \cA$, we must have
that the total variation distance between $\cD$ and $\cD+v$, denoted by
$\tv(\cD, \cD+v)$, is bounded by $\delta$. The rest of our results will be
consequences of this approximate translation-invariance property.
\begin{restatable}{lemma}{lemRobustToTV} \label{lem:robustToTV}
  Let $\cD$ be a distribution on $\reals^d$ such that for every (randomized)
  classifier $f : \reals^d \to \cY$, the smoothed classifier $g(\cdot; \cD, f)$
  is $(\cA, \delta)$-robust. Then for all $v \in \cA$, we have $\tv(\cD, \cD+v)
  \leq \delta$.
\end{restatable}

\paragraph{Lower bound on noise magnitude.} Our first result is a lower bound on
the expected squared $\ell_2$-magnitude of a sample $\eta \sim \cD$ for any distribution
$\cD$ that is approximately invariant to $\ell_p$-translations of size
$\epsilon$.

\begin{restatable}{theorem}{thmLowerBound}\label{thm:lowerBound}
  Fix any $p \geq 2$ and let $\cD$ be a distribution on $\reals^d$ such that
  there exists a radius $\epsilon$ and total variation bound $\delta$ satisfying
  that for all $v \in \reals^d$ with $\norm{v}_p \leq \epsilon$ we have
  $\tv(\cD, \cD+v) \leq \delta$. Then
  \[
  \expect_{\eta \sim \cD} \norm{\eta}_2^2
  \geq \frac{\epsilon^2 d^{2 - 2/p}}{800} \cdot \frac{1 - \delta}{\delta^2}.
  \]
\end{restatable}

As a consequence of \Cref{thm:lowerBound} and Lemma \ref{lem:robustToTV}, it follows that any distribution that
ensures $(\epsilon, \delta)$-robustness with respect to $\norm{\cdot}_p$ for any
base classifier $f$ must also satisfy the same lower bound.

\paragraph{Phase transition at $p = 2$.} The lower bound given by
\Cref{thm:lowerBound} implies a phase transition in the nature of distributions
$\cD$ that are able to ensure $(\epsilon, \delta)$-robustness with respect to
$\norm{\cdot}_p$ that occurs at $p = 2$. For $p \leq 2$, the necessary expected squared
$\ell_2$-magnitude of a sample from $\cD$ grows only like $\sqrt{d}$, which is
consistent with adding a constant level of noise to every feature in the input
example (e.g., as would happen when using a Gaussian distribution with standard
deviation $\sigma = \frac{2\epsilon}{\delta}$). On the other hand, for $p > 2$,
the expected $\ell_2$ magnitude of a sample from $\cD$ grows strictly faster
than $\sqrt{d}$, which, intuitively, requires that the noise added to each
component of the input example must scale with the input dimension $d$, rather
than remaining constant as in the $p \leq 2$ regime. More formally, we prove the
following:

\begin{restatable}[hardness of random smoothing]{theorem}{thmComponentProperties}\label{thm:componentProperties}
  Fix any $p > 2$ and let $\cD$ be a distribution on $\reals^d$ such that
  for every (randomized)
  classifier $f : \reals^d \to \cY$, the smoothed classifier $g(\cdot; \cD, f)$
  is $(\cA, \delta)$-robust.
  Let $\eta$ be a sample from $\cD$. Then at least 99\% of the
  components of $\eta$ satisfy $\expect \eta_i^2 = \Omega(\frac{\epsilon^2
  d^{1-2/p}(1-\delta)}{\delta^2})$. Moreover, if $\cD$ is a product measure of i.i.d.
  noise (i.e., $\cD = (\cD')^d$), then the tail of $\cD'$ satisfies
  $\prob_{\zeta \sim \cD'}(|\zeta| > s) \geq \left( \frac{c \epsilon(1-\delta)}{s \delta}
  \right)^{2p/(p-2)}$ for some $s > c\epsilon(1-\delta) / \delta$, where $c$ is an
  absolute constant. In other words, $\cD'$ is a heavy-tailed
  distribution.\footnote{A distribution is heavy-tailed, if its tail is not an
  exponential function of $x$ for all $x>0$ (i.e., not an sub-exponential or
  sub-Gaussian distribution)~\citep{vershynin2018high}.}
\end{restatable}

The phase transition at $p = 2$ is more clearly evident from
\Cref{thm:componentProperties}. In particular, the variance of most components
of the noise must grow with $d^{1-2/p}$. \Cref{thm:componentProperties} shows
that \emph{any} distribution that provides $(\epsilon, \delta)$-robustness with
respect to $\norm{\cdot}_p$ for $p > 2$ must have very high variance in most of
its component distributions when the dimension $d$ is large. In particular, for
$p = \infty$ the variance grows linearly with the dimension. Similarly, if we
use a product distribution to achieve $(\epsilon, \delta)$-robustness with
respect to $\norm{\cdot}_p$ with $p > 2$, then each component of the noise
distribution must be heavy-tailed and is likely to generate very large
perturbations.

\subsection{Technical overview}

\medskip
\noindent{\textbf{Total-variation bound of noise magnitude.}} Our
results
demonstrate
a strong connection between the required noise magnitude
$\expect\|\eta\|_2^2$ in random smoothing and the total variation distance between $\cD$
and its shifted distribution $\cD+v$ in the worst-case direction $v$. The total variation distance has a very natural explanation on the hardness
of testing $\cD$ v.s. $\cD+v$: any classifier cannot distinguish $\cD$ from
$\cD+v$ with a good probability related to $\tv(\cD,\cD +v)$. Our analysis applies the following techniques.



\medskip
\noindent{\textbf{Warm-up: one-dimensional case.}} We begin our
analysis of Theorem \ref{thm:lowerBound} with the one-dimension case, by studying
the projection of noise $\eta\in\reals^d$ on a direction $v\in\reals^d$. A
simple use of Chebyshev's inequality implies $\expect_{\eta \sim \cD} |v^\top \eta|^2
\geq \norm{v}^2_4 (1 -  \delta) / 8$. To see this, let $\eta$ be a sample from
$\cD$ and let $\eta' = \eta + v$ so that $\eta'$ is a sample from $\cD + v$.
Define $Z = v^\top \eta$ and $Z' = v^\top \eta' = Z + \norm{v}_2^2$. Define $r =
\norm{v}_2^2 / 2$ so that the intervals $\mathcal{A} = (-r, r)$ and $\mathcal{B} = [\norm{v}_2^2 -
r, \norm{v}_2^2 + r]$ are disjoint. From Chebyshev's inequality, we have $\prob(Z
\in \mathcal{A}) \geq 1 - \expect |Z|^2 / r^2$. Similarly, $\prob(Z' \in \mathcal{B}) \geq 1 - \expect
|Z|^2 / r^2$ and, since $\mathcal{A}$ and $\mathcal{B}$ are disjoint, this implies $\prob(Z' \in \mathcal{A}) <
\expect |Z|^2 / r^2$. Therefore, $\tv(\cD, \cD+v) \geq \prob(Z \in \mathcal{A}) - \prob(Z' \in
\mathcal{A}) \geq 1 - 2 \expect |Z|^2 / r^2$. The claim follows from rearranging this
inequality and the fact $\delta\ge \tv(\cD, \cD+v)$.

The remainder of the one-dimension case is to show $\expect_{\eta \sim \cD} |v^\top
\eta| \geq \norm{v}^2_2 \frac{(1-\delta)^2}{8\delta}$. To this end, we exploit a
nice property of total variation distance in $\reals$: every $\epsilon$-interval
$I = [a, a+\epsilon)$ satisfies $\cD(I) \leq \tv(\cD, \cD+\epsilon)$. We note that for any $\tau \geq 0$, rearranging
Markov's inequality gives $\expect |v^\top\eta| \geq \tau \prob(|v^\top\eta| >
\tau) = \tau(1 - \prob(|v^\top\eta| \leq \tau))$. We can cover the set $\{x \in
\reals \,:\, |x| \leq \tau\}$ using $\lceil \frac{2\tau}{\epsilon} \rceil$
intervals of width $\epsilon=\|v\|_2^2$ and, by this property, each of those
intervals has probability mass at most $\delta$. It follows that
$\prob(|v^\top\eta| \leq \tau) \leq \lceil \frac{2\tau}{\epsilon} \rceil
\delta$, implying $\expect |v^\top\eta| \geq \tau(1 - \lceil
\frac{2\tau}{\epsilon} \rceil \delta)$. Finally, we optimize $\tau$ to obtain
the bound $\expect_{\eta\sim\cD} |v^\top\eta| \geq
\|v\|_2^2\frac{(1-\delta)^2}{8\delta}$, as desired.

\medskip
\noindent{\textbf{Extension to the $d$-dimensional case.}} A bridge to connect
one-dimensional case with $d$-dimensional case is the Pythagorean theorem: if
there exists a set of orthogonal directions $v_i$'s such that $\expect_{\eta
\sim \cD} |v_i^\top \eta|^2 \geq \frac{\norm{v_i}^4_2}{200}
\frac{1-\delta}{\delta^2}$ and $\|v_i\|_2=\epsilon d^{1/2-1/p}$ (the furthest
distance to $x$ in the $\ell_p$ ball $\ball{p}(x,\epsilon)$), the Pythagorean
theorem implies the result for the $d$-dimensional case straightforwardly. The
existence of a set of orthogonal directions that satisfy these requirements is
easy to find for the $\ell_2$ case, because the $\ell_2$ ball is isotropic and
any set of orthogonal bases of $\reals^d$ satisfies the conditions. However, the
problem is challenging for the $\ell_p$ case, since the $\ell_p$ ball is not
isotropic in general. In Corollary \ref{cor:badDirections}, we show that there
exist at least $d/2$ $v_i$'s which satisfy the requirements. Using the
Pythagorean theorem in the subspace spanned by such $v_i$'s gives Theorem
\ref{thm:lowerBound}.

\medskip
\noindent{\textbf{Peeling argument and tail probability.}} We now summarize our
main techniques to prove Theorem \ref{thm:componentProperties}. By
$\|\eta\|_2\le\sqrt{d}\|\eta\|_\infty$, Theorem \ref{thm:lowerBound} implies
$\expect \eta_i^2 \geq \frac{\epsilon d^{1/2 - 1/p}}{800}\cdot
\frac{1-\delta}{\delta^2}$ for at least one index $i$, which shows that at least one component of $\eta$ is
large. However, this guarantee only highlights the largest pixel of $|\eta|$.
Rather than working with the $\ell_\infty$-norm of $\eta$, we apply a similar
argument to show that the variance of at least one component of $\eta$ must be
large. Next, we consider the $(d-1)$-dimensional distribution obtained by removing
the highest-variance feature. Applying an identical argument, the highest-variance remaining feature must also be large. Each time we repeat this
procedure, the strength of the variance lower bound decreases since the
dimensionality of the distribution is decreasing. However, we can apply this
peeling strategy for any constant fraction of the components of $\eta$ to obtain
lower bounds. The tail-probability guarantee in Theorem
\ref{thm:componentProperties} follows a standard moment analysis in
\citep{vershynin2018high}.

\medskip
\noindent{\textbf{Summary of our techniques.}} Our proofs---in
particular, the use of the Pythagorean theorem---show that defending against
adversarial attacks in the $\ell_p$ ball of radius $\epsilon$ by random smoothing is almost as hard
as defending against attacks in the $\ell_2$ ball of radius $\epsilon
d^{1/2-1/p}$. Therefore, the $\ell_\infty$ certification procedure---firstly
using Gaussian smoothing to certify $\ell_2$ robustness and then dividing the
$\ell_2$ certified radius by $\sqrt{d}$ as in \citep{salman2019provably}---is
almost an optimal random smoothing approach for certifying $\ell_\infty$
robustness. The principle might hold generally for other threat models beyond
$\ell_p$ robustness, and sheds light on the design of new random smoothing and
proofs of hardness in the other threat models broadly.

\section{Related Works} \label{sec:relatedwork}

\comment{
\medskip
\noindent{\textbf{Goal and algorithm.}}
In the algorithm of random smoothing, we are given an instance $x\in\reals^d$, a classifier $f:\reals^d\rightarrow \mathcal{Y}$, and a radius $R$. The goal is to certify that the prediction of the classifier $f$ is constant within an $\ell_p$ ball of radius $R$ centered at $x$. The algorithm proceeds by repeatedly adding independent random noises $\{\eta_i\}_{i=1}^m$ to $x$, calculating $f(x+\eta_i)$, and outputting the class according to the majority vote, that is, the most-likely class by $\{f(x+\eta_i)\}_{i=1}^m$.

\medskip
\noindent{\textbf{Related works.}}
}

\medskip
\noindent{\textbf{$\ell_2$ robustness.}}
Probably one of the most well-understood results for random smoothing is the $\ell_2$ robustness. With Gaussian random noises, \citet{lecuyer2019certified} and \citet{li2019certified} provided the first guarantee of random smoothing and was later improved by \citet{cohen2019certified} with the following theorem.
\begin{theorem}[Theorem 1 of \citet{cohen2019certified}]
\label{thm: cohen's result}
Let $f:\reals^d\rightarrow \mathcal{Y}$ by any deterministic or random classifier, and let $\eta\sim\normal(0,\sigma^2 I)$. Let $g(x)=\argmax_{c\in\mathcal{Y}} \Pr(f(x+\eta)=c)$. Suppose $c_A\in\mathcal{Y}$ and $\underline{p_A},\overline{p_B}\in[0,1]$ satisfy:
$
\Pr(f(x+\eta)=c_A)\ge \underline{p_A}\ge \overline{p_B}\ge \max_{c\not= c_A} \Pr(f(x+\eta)=c).
$
Then $g(x+\delta)=c_A$ for all $\|\delta\|_2<\epsilon$, where
$\epsilon=\frac{\sigma}{2}(\Phi^{-1}(\underline{p_A})-\Phi^{-1}(\overline{p_B}))$,
and $\Phi(\cdot)$ is the cumulative distribution function of standard Gaussian distribution.
\end{theorem}

Note that Theorem \ref{thm: cohen's result} holds for \emph{arbitrary}
classifier. Thus a hardness result of random smoothing---the one in an opposite
direction of Theorem \ref{thm: cohen's result}---requires finding a hard
instance of classifier $f$ such that a similar conclusion of Theorem
\ref{thm: cohen's result} does not hold, i.e., the resulting smoothed classifier
$g$ is trivial as the noise variance is too large. Our results of Theorems
\ref{thm:lowerBound} and \ref{thm:componentProperties} are in such flavour. Beyond the top-$1$ predictions in Theorem \ref{thm: cohen's result},
\citet{jia2020certified} studied the certified robustness for top-$k$
predictions via random smoothing under Gaussian noise and derive a tight
robustness bound in $\ell_2$ norm. In this paper, however, we study the standard
setting of top-$1$ predictions.

\medskip
\noindent{\textbf{$\ell_p$ robustness.}}
Beyond the $\ell_2$ robustness, random smoothing also achieves the state-of-the-art certified $\ell_p$ robustness for $p<2$. \citet{lee2019tight} provided adversarial robustness guarantees and associated random-smoothing algorithms for the discrete case where the adversary is $\ell_0$ bounded. \citet{li2019certified} suggested replacing Gaussian with Laplacian noise for the $\ell_1$ robustness. \citet{dvijotham2020a} introduced a general framework for proving robustness properties of smoothed classifiers in the black-box setting using $f$-divergence. However, much remains unknown concerning the effectiveness of random smoothing for $\ell_p$ robustness with $p>2$. \citet{salman2019provably} proposed an algorithm for certifying $\ell_\infty$ robustness, by firstly certifying $\ell_2$ robustness via the algorithm of \citet{cohen2019certified} and then dividing the certified $\ell_2$ radius by $\sqrt{d}$. However, the certified $\ell_\infty$ radius by this procedure is as small as $\mathcal{O}(1/\sqrt{d})$, in contrast to the constant certified radius as discussed in this paper.

\medskip
\noindent{\textbf{Training algorithms.}}
While random smoothing certifies inference-time robustness for any given base classifier $f$, the certified robust radius might vary a lot for different training methods. This motivates researchers to design new training algorithms of $f$ that particularly adapts to random smoothing. \citet{zhai2020macer} trained a robust smoothed classifier via maximizing the certified radius. In contrast to using naturally trained classifier in \citep{cohen2019certified}, \citet{salman2019provably} combined adversarial training of \citet{madry2017towards} with random smoothing in the training procedure of $f$. In our experiment, we introduce a new baseline which combines TRADES~\citep{zhang2019theoretically} with random smoothing to train a robust smoothed classifier.

\section{Analysis of Main Results} \label{sec:theory}
In this section we prove \Cref{thm:lowerBound} and
\Cref{thm:componentProperties}.

\subsection{Analysis of Theorem \ref{thm:lowerBound}}

In this section we prove \Cref{thm:lowerBound}. Our proof has two main
steps: first, we study the one-dimensional version of the problem and prove two
complementary lower bounds on the magnitude of a sample $\eta$ drawn from a
distribution $\cD$ over $\reals$ with the property that for all $v \in \reals$
with $|v| \leq \epsilon$ we have $\tv(\cD, \cD + v) \leq \delta$. Next, we show
how to apply this argument to $\Omega(d)$ orthogonal 1-dimensional subspaces in
$\reals^d$ to lower bound the expected magnitude of a sample drawn from a
distribution $\cD$ over $\reals^d$, with the property that for all $v \in
\reals^d$ with $\norm{v}_p \leq \epsilon$, we have $\tv(\cD, \cD + v) \leq
\delta$.

\paragraph{One-dimensional results.} Our first result lower bounds the magnitude
of a sample from any distribution $\cD$ in terms of the total variation distance
between $\cD$ and $\cD + \epsilon$ for any $\epsilon \geq 0$.

\begin{lemma}\label{lem: one dimensional lower bound}
  Let $\cD$ be any distribution on $\reals$, $\eta$ be a sample from $\cD$,
  $\epsilon \geq 0$, and let $\delta = \tv(\cD, \cD + \epsilon)$. Then we have\footnote{We do not try to optimize constants throughout the paper.}
  \[
  \expect |\eta|^2
  \geq \frac{\epsilon^2}{200} \cdot \frac{1-\delta}{\delta^2}.
  \]
\end{lemma}

We prove Lemma \ref{lem: one dimensional lower bound} using two complementary lower
bounds. The first lower bound is tighter for large $\delta$, while the second
lower bound is tighter when $\delta$ is close to zero. Taking the maximum of the
two bounds proves Lemma \ref{lem: one dimensional lower bound}.

\begin{lemma}\label{lem: one dimensional lower bound big delta}
  Let $\cD$ be any distribution on $\reals$, $\eta$ be a sample from $\cD$,
  $\epsilon \geq 0$, and let $\delta = \tv(\cD, \cD + \epsilon)$. Then we have
  \[
  \expect |\eta|^2
  \geq \frac{\epsilon^2}{8} \cdot (1-\delta).
  \]
\end{lemma}
\begin{proof}
  Let $\eta' = \eta + \epsilon$ so that $\eta'$ is a sample from $\cD +
  \epsilon$ and define $r = \epsilon / 2$ so that the sets $\mathcal{A} =
  (-r,r)$ and $\mathcal{B} = [\epsilon - r, \epsilon + r]$ are disjoint. From
  Chebyshev's inequality, we have that $\prob(\eta \in \mathcal{A}) = 1 -
  \prob(|\eta| \geq r) \geq 1 - \frac{\expect |\eta|^2}{r^2}$. Further, since $\eta'
  \in \mathcal{B}$ if and only if $\eta \in \mathcal{A}$, we have $\prob(\eta'
  \in \mathcal{B}) \geq 1 - \frac{\expect |\eta|^2}{r^2}$. Next, since $\mathcal{A}$
  and $\mathcal{B}$ are disjoint, it follows that $\prob(\eta' \in \mathcal{A})
  \leq 1 - \prob(\eta' \in \mathcal{B}) \leq 1 - 1 + \frac{\expect |\eta|^2}{r^2} =
  \frac{\expect |\eta|^2}{r^2}$. Finally, we have $\delta \geq \prob(\eta \in
  \mathcal{A}) - \prob(\eta' \in \mathcal{A}) \geq 1 - \frac{2 \expect
  |\eta|^2}{r^2} = 1 - \frac{8 \expect |\eta|^2}{\epsilon^2}$. Rearranging this
  inequality proves the claim.
\end{proof}

Next, we prove a tighter bound when $\delta$ is close to zero. The key insight
is that no interval $I \subseteq \reals$ of width $\epsilon$ can have probability
mass larger than $\tv(\cD, \cD+\epsilon)$. This implies that the mass of $\cD$
cannot concentrate too close to the origin, leading to lower bounds on the
expected magnitude of a sample from $\cD$.

\begin{lemma}\label{lem: one dimensional lower bound small delta}
  Let $\cD$ be any distribution on $\reals$, $\eta$ be a sample from $\cD$,
  $\epsilon \geq 0$, and let $\delta = \tv(\cD, \cD + \epsilon)$. Then we have
  \[
  \expect |\eta|
  \geq \frac{\epsilon}{8} \cdot \frac{(1-\delta)^2}{\delta},
  \]
  which implies $\expect |\eta|^2\ge \frac{\epsilon^2}{64} \cdot \frac{(1-\delta)^4}{\delta^2}$.
\end{lemma}
\begin{proof}
  The key step in the proof is to show that every interval $\mathcal{I} = [a, a+\epsilon)$
  of length $\epsilon$ has probability mass at most $\delta$ under the
  distribution $\cD$. Once we have established this fact, then the proof is as
  follows: for any $\tau \geq 0$, rearranging Markov's inequality gives $\expect
  |\eta| \geq \tau \prob(|\eta| > \tau) = \tau(1 - \prob(|\eta| \leq \tau))$. We
  can cover the set $\{x \in \reals \,:\, |x| \leq \tau\}$ using $\lceil
  \frac{2\tau}{\epsilon} \rceil$ intervals of width $\epsilon$ and each of those
  intervals has probability mass at most $\delta$. It follows that $\prob(|\eta|
  \leq \tau) \leq \lceil \frac{2\tau}{\epsilon} \rceil \delta$, implying
  $\expect |\eta| \geq \tau(1 - \lceil \frac{2\tau}{\epsilon} \rceil \delta)$.
  Since $\lceil \frac{2\tau}{\epsilon} \rceil \leq \frac{2\tau}{\epsilon} + 1$,
  we have $\expect |\eta| \geq (1 - \delta)\tau - \frac{2\delta}{\epsilon}
  \tau^2$. Finally, we optimize $\tau$ to get the strongest bound. The strongest
  bound is obtained at $\tau = \frac{\epsilon(1-\delta)}{4\delta}$, which gives
  $\expect |\eta| \geq \frac{\epsilon(1-\delta)^2}{8\delta}$.

  It remains to prove the claim that all intervals of length $\epsilon$ have
  probability mass at most $\delta$. Let $\mathcal{I} = [a, a+\epsilon)$ be any such
  interval. The proof has two steps: first, we partition $\reals$ using a
  collection of translated copies of the interval $\mathcal{I}$, and show that the
  difference in probability mass between any pair of intervals in the partition
  is at most $\delta$. Then, given that there must be intervals with probability
  mass arbitrarily close to zero, this implies that the probability mass of any
  interval (and in particular, the probability mass of $\mathcal{I}$) is upper bounded by
  $\delta$.

  For each integer $i \in \integers$, let $\mathcal{I}_i = \mathcal{I} + i\epsilon = \{x + i\epsilon: x \in \mathcal{I}\}$ be a copy of the interval $\mathcal{I}$ translated by $i \epsilon$. By
  construction the set of intervals $\mathcal{I}_i$ for $i \in \integers$ forms a
  partition of $\reals$. For any indices $i < j$, we can express the difference
  in probability mass between $\mathcal{I}_i$ and $\mathcal{I}_j$ as a telescoping sum: $\cD(\mathcal{I}_j) -
  \cD(\mathcal{I}_i) = \sum_{k = i}^{j-1} [\cD(\mathcal{I}_{k+1}) - \cD(\mathcal{I}_{k})]$. We will show that
  for any $i < j$, the telescoping sum is contained in $[-\delta,\delta]$. Let
  $P = \{ k \in (i,j]: \cD(\mathcal{I}_{k+1}) - \cD(\mathcal{I}_k) > 0\}$ be the indices of the
  positive terms in the sum. Then, since the telescoping sum is upper bounded by
  the sum of its positive terms and the intervals are disjoint, we have
  $$
  \cD(\mathcal{I}_j) - \cD(\mathcal{I}_i)
  \leq \sum_{k \in P} [\cD(\mathcal{I}_{k+1}) - \cD(\mathcal{I}_k)]
  = \cD\left( \bigcup_{k \in P} \mathcal{I}_{k+1} \right) - \cD\left( \bigcup_{k \in P} \mathcal{I}_k \right).
  $$
  For all $k \in P$ we have $\eta \in \mathcal{I}_k$ if and only if $\eta + \epsilon \in
  \mathcal{I}_{k+1}$, which implies $\prob(\eta \in \bigcup_{k \in P} \mathcal{I}_k) = \prob(\eta +
  \epsilon \in \bigcup_{k \in P} \mathcal{I}_{k+1})$. Combined with the definition of the
  total variation distance, it follows that
  $
  \cD\left( \bigcup_{k \in P} \mathcal{I}_{k+1} \right) - \cD\left( \bigcup_{k \in P} \mathcal{I}_k \right)
  = \prob\left(\eta \in \bigcup_{k \in P} \mathcal{I}_{k+1} \right)
  - \prob\left(\eta \in \bigcup_{k \in P} \mathcal{I}_{k} \right)
  = \prob\left(\eta \in \bigcup_{k \in P} \mathcal{I}_{k+1} \right)
  - \prob\left(\eta + \epsilon \in \bigcup_{k \in P} \mathcal{I}_{k+1} \right) \leq \delta,
  $
  and therefore $\cD(\mathcal{I}_j) - \cD(\mathcal{I}_i) \leq \delta$. A similar argument applied to
  the negative terms of the telescoping sum guarantees that $\cD(\mathcal{I}_j) - \cD(\mathcal{I}_i)
  \geq -\delta$, proving that $|\cD(\mathcal{I}_j) - \cD(\mathcal{I}_i)| \leq \delta$.

  Finally, for any $\alpha > 0$, there must exist an interval $\mathcal{I}_j$ such that
  $\cD(\mathcal{I}_j) < \alpha$ (since otherwise the total probability mass of all the
  intervals would be infinite). Since no pair of intervals in the partition can
  have probability masses differing by more than $\delta$, this implies that
  $\cD(\mathcal{I}) \leq \alpha + \delta$ for any $\alpha$. Taking the limit as $\alpha
  \to 0$ shows that $\cD(\mathcal{I}) \leq \delta$, completing the proof.
\end{proof}

Finally, Lemma \ref{lem: one dimensional lower bound} follows from Lemmas \ref{lem: one dimensional lower bound big
delta} and \ref{lem: one
dimensional lower bound small delta}, and the fact that for any $\delta \in (0,1]$, we have
$\max\{\frac{1-\delta}{8}, \frac{(1-\delta)^4}{64\delta^2}\} \geq \frac{1}{200}
\cdot \frac{1-\delta}{\delta^2}$.

\paragraph{Extension to the $d$-dimensional case.} For the remainder of this
section we turn to the analysis of distributions $\cD$ defined over $\reals^d$.
First, we use Lemma \ref{lem: one dimensional lower bound} to lower bound the
magnitude of noise drawn from $\cD$ when projected onto any one-dimensional
subspace.

\begin{corollary}\label{cor:dirMoment}
  Let $\cD$ be any distribution on $\reals^d$, $\eta$ be a sample from $\cD$, $v
  \in \reals^d$, and let $\delta = \tv(\cD, \cD + v)$. Then we have
  \[
  \expect_{\eta \sim \cD} \frac{|v^\top \eta|^2}{\norm{v}_2^2}
  \geq \frac{\norm{v}_2^2}{200} \cdot \frac{1 - \delta}{\delta^2}.
  \]
\end{corollary}
\begin{proof}
  Let $\eta$ be a sample from $\cD$, $\eta' = \eta + v$ be a sample from $\cD +
  v$, and define $Z = v^\top \eta$ and $Z' = v^\top \eta' = Z + \norm{v}_2^2$.
  Then the total variation distance between $Z$ and $Z'$ is bounded by $\delta$,
  and $Z'$ corresponds to a translation of $Z$ by a distance $\norm{v}_2^2$.
  Therefore, applying Lemma \ref{lem: one dimensional lower bound} with $\epsilon =
  \norm{v}_2^2$, we have that $\expect |v^\top \eta|^2 = \expect |Z|^2 \geq
  \norm{v}_2^4 \cdot \frac{1-\delta}{200\delta^2}$. Rearranging this inequality
  completes the proof.
\end{proof}

Intuitively, Corollary \ref{cor:dirMoment} shows that for any vector $v \in \reals^d$
such that $\tv(\cD, \cD + v)$ is small, the expected magnitude of a sample $\eta
\sim \cD$ when projected onto $v$ cannot be much smaller than the length of $v$.
The key idea for proving \Cref{thm:lowerBound} is to construct a large number of
orthogonal vectors $v_1, \dots, v_b$ with small $\ell_p$ norms but large
$\ell_2$ norms. Then $\cD$ will have to be ``spread out'' in all of these
directions, resulting in a large expected $\ell_2$ norm. We begin by showing
that whenever $d$ is a power of two, we can find an orthogonal basis for
$\reals^d$ in $\{\pm 1\}^d$.

\begin{lemma} \label{lem:orthogonalCorners}
  For any $n \geq 0$ there exist $d = 2^n$ orthogonal vectors $v_1, \dots, v_d
  \in \{ \pm 1 \}^d$.
\end{lemma}
\begin{proof}
  The proof is by induction on $n$. For $n = 0$, we have $d = 1$ and the vector
  $v_1 = (1)$ satisfies the requirements. Now suppose the claim holds for $n$
  and let $v_1, \dots, v_d$ be orthogonal in $\{\pm 1 \}^d$ for $d = 2^n$. For
  each $i \in [d]$, define $a_i = (v_i, v_i) \in \{\pm 1\}^{2d}$ and $b_i =
  (v_i, -v_i) \in \{\pm 1\}^{2d}$. We will show that these vectors are
  orthogonal. For any indices $i$ and $j$, we can compute the inner products
  between pairs of vectors among $a_i$, $a_j$, $b_i$, and $b_j$: $a_i^\top a_j =
  2 v_i^\top v_j$, $b_i^\top b_j = 2 v_i^\top v_j$, and $a_i^\top b_j = v_i^\top
  v_j - v_i^\top v_j = 0$. Therefore, for any $i \neq j$, since $v_i^\top v_j =
  0$, we are guaranteed that $a_i^\top a_j = 0$, $b_i^\top b_j = 0$, and
  $a_i^\top b_j = 0$. It follows that the $2^{d+1}$ vectors $a_1, \dots, a_d,
  b_1, \dots, b_d$ are orthogonal.
\end{proof}

From this, it follows that for any dimension $d$, we can always find a
collection of $b \geq d/2$ vectors that are short in the $\ell_p$ norm, but long
in the $\ell_2$ norm. Intuitively, these vectors are the vertices of a hypercube
in a $b$-dimensional subspace. \Cref{fig:lpBall} depicts the construction.

\begin{figure}
  \centering
  \includegraphics[width=0.26\columnwidth]{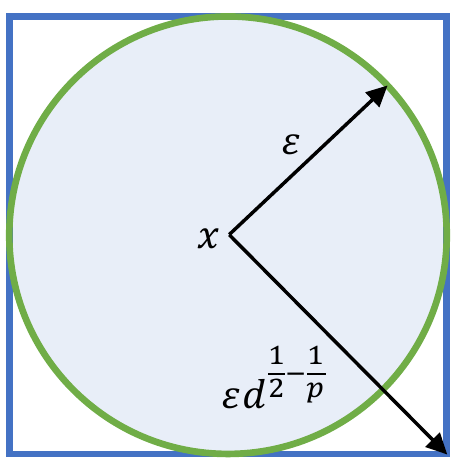}
  \caption{Vectors pointing towards the corner of the cube in $\reals^d$ have
  large $\ell_2$ norm but small $\ell_p$ norm.}\label{fig:lpBall}
  \vspace{-0.5cm}
\end{figure}

\begin{corollary}\label{cor:badDirections}
  For any $p \geq 2$ and dimension $d$, there exist $b \geq d/2$ orthogonal
  vectors $v_1, \dots, v_b \in \reals^d$ such that $\norm{v_i}_2 = b^{1/2 - 1/p}
  \geq (d/2)^{1/2 - 1/p}$ and $\norm{v_i}_p = 1$ for all $i \in [b]$. This holds
  even when $p = \infty$.
\end{corollary}
\begin{proof}
  Let $n$ be the largest integer such that $2^n \leq d$. We must have $2^n >
  d/2$, since otherwise $2^{n+1} \leq d$. We now apply Lemma \ref{lem:orthogonalCorners}
  to find $b = 2^n$ orthogonal vectors $u_1, \dots, u_b \in \{\pm 1\}^{b}$. For
  each $i \in [b]$, we have that $\norm{u_i}_p = b^{1/p}$. Finally, for $i \in
  [b]$, define $v_i = (u_i \cdot b^{-1/p}, 0, \dots, 0) \in \reals^d$ to be a
  normalized copy of $u_i$ padded with $d - b$ zeros. For all $i \in [b]$, we
  have $\norm{v_i}_p = 1$ and $\norm{v_i}_2 = (b \cdot b^{-2/p})^{1/2} = b^{1/2
  - 1/p} \geq (d/2)^{1/2 - 1/p}$.
\end{proof}

With this, we are ready to prove \Cref{thm:lowerBound}.

\thmLowerBound*
\begin{proof}
  Let $\eta$ be a sample from $\cD$. By scaling the vectors from
  Corollary \ref{cor:badDirections} by $\epsilon$, we obtain $b > d/2$ vectors $v_1,
  \dots, v_b \in \reals^d$ with $\norm{v_i}_p = \epsilon$ and $\norm{v_i}_2 =
  \epsilon \cdot b^{1/2 - 1/p}$. By assumption we must have $\tv(\cD, \cD + v_i)
  \leq \delta$, since $\norm{v_i}_p \leq \epsilon$, and Corollary \ref{cor:dirMoment}
  implies that $\expect \frac{|v_i^\top \eta|^2}{\norm{v_i}_2^2} \geq
  \frac{\norm{v_i}_2^2}{200} \frac{1-\delta}{\delta^2}$ for each $i$. We use
  this fact to bound $\expect \norm{\eta}_2^2$.

  Let $\bQ \in \reals^{b \times d}$ be the matrix whose $i^\text{th}$ row is
  given by $v_i / \norm{v_i}_2$ so that $\bQ$ is the orthogonal projection
  matrix onto the subspace spanned by the vectors $v_1, \dots, v_b$. Then we
  have
  $ \expect \norm{\eta}_2^2 \geq \expect \norm{\bQ \eta}_2^2=  \sum_{i=1}^b \expect
  \frac{|v_i^\top \eta|^2}{\norm{v_i}_2^2}\geq  \sum_{i=1}^b
  \frac{\norm{v_i}_2^2}{200} \frac{1-\delta}{\delta^2}$,
  where the first inequality follows because orthogonal projections are
  non-expansive, the equality follows from the Pythagorean theorem, and the last inequality follows from Corollary \ref{cor:dirMoment}.
  Using the fact that $\norm{v_i}_2 = \epsilon \cdot b^{1/2 - 1/p}$, we have
  that $\expect \norm{\eta}_2^2 \geq \frac{\epsilon^2 b^{2 - 2/p}}{200} \cdot \frac{1
  - \delta}{\delta^2}$. Finally, since $b > d/2$ and $(1/2)^{2 - 2/p} \geq 1/4$
  for $p \geq 2$, we have $\expect \norm{\eta}_2^2 \geq \frac{\epsilon^2 d^{2 -
  2/p}}{800} \cdot \frac{1 - \delta}{\delta^2}$, as required.
\end{proof}

\subsection{Analysis of Theorem \ref{thm:componentProperties}}

In this section we prove the variance and heavy-tailed properties from
\Cref{thm:componentProperties} separately.

Combining \Cref{thm:lowerBound} with a peeling argument, we are able to lower
bound the marginal variance in most of the coordinates of $\eta$.

\begin{lemma}\label{lem:coordinateBounds}
  Fix any $p \geq 2$ and let $\cD$ be a distribution on $\reals^d$ such that
  there exists a radius $\epsilon$ and total variation bound $\delta$ so that
  for all $v \in \reals^d$ with $\norm{v}_p \leq \epsilon$ we have $\tv(\cD,
  \cD+v) \leq \delta$. Let $\eta$ be a sample from $\cD$ and $\sigma$ be the
  permutation of $[d]$ such that $\expect[\eta_{\sigma(1)}^2] \geq \dots \geq
  \expect[\eta_{\sigma(d)}^2]$. Then for any $i \in [d]$, we have
  $$\expect[\eta_{\sigma(i)}^2] \geq \frac{\epsilon^2 (d-i+1)^{1 - 2/p}}{800}
  \frac{1 - \delta}{\delta^2}.$$
\end{lemma}
\begin{proof}
  For each index $i$, let $P_i : \reals^d \to \reals^{d - i + 1}$ be the
  projection $P_i(x) = (x_{\sigma(i)}, x_{\sigma(i+1)}, \dots,
  x_{\sigma(d)})$ and $\cD_i$ be the distribution of $P_i(\eta)$. First we argue
  that for each $i \in [d]$ and any $v \in \reals^{d-i+1}$ with $\norm{v}_p \leq
  \epsilon$, we must have $\tv(\cD_i, \cD_i + v) \leq \delta$. To see this, let
  $z \in \reals^d$ be the vector such that $P_i(z) = v$ and $z_{\sigma(1)} =
  \dots = z_{\sigma(i-1)} = 0$. Then
  $
  \tv(\cD_i, \cD_i + v)
  = \sup_{\cA \subseteq \reals^{d-i+1}} |\prob(P_i(\eta) \in \cA) - \prob(P_i(\eta) + v \in \cA)|
  = \sup_{\cA \subseteq \reals^{d-i+1}} |\prob(P_i(\eta) \in \cA) - \prob(P_i(\eta + z) \in \cA)|
  \leq \sup_{\mathcal{A} \subseteq \reals^d} |\prob(\eta \in \mathcal{A}) - \prob(\eta + z \in \mathcal{A})|= \tv(\cD, \cD + z)
  $.
  Next, since $\norm{z}_p = \norm{v}_p \leq \epsilon$, we must have $\tv(\cD,
  \cD + v) \leq \delta$.

  Now fix an index $i \in [d]$ and let $Z$ be a sample from $\cD_i$. Applying
  \Cref{thm:lowerBound} to $Z$, we have that $\expect \norm{Z}_2^2 \geq
  \frac{\epsilon^2 (d-i+1)^{2- 2/p}}{800} \cdot \frac{1-\delta}{\delta^2}.$
  Since there must exist at least one index $l$ such that $\expect[Z_l^2] \geq
  \frac{1}{d-i+1}\sum_{j=1}^{d-i+1} \expect[Z_j^2]$, it follows that at least
  one coordinate $l$ must satisfy $\expect[Z_l^2] \geq \frac{\epsilon^2
  (d-i+1)^{1 - 2/p}}{800} \cdot \frac{1 - \delta}{\delta^2}.$
  Finally, since the coordinates of $Z$ are the $(d-i+1)$ coordinates of $\eta$
  with the smallest variance, it follows that $$\expect[\eta_{\sigma(i)}^2] \geq
  \frac{\epsilon^2 (d-i+1)^{1 - 2/p}}{800} \cdot \frac{1 -
  \delta}{\delta^2},$$ as required.
\end{proof}

Lemma \ref{lem:coordinateBounds} implies that any distribution $\cD$ over $\reals^d$
such that for all $v \in \reals^d$ with $\norm{v}_p \leq \epsilon$ we have
$\tv(\cD, \cD + v) \leq \delta$ for $p > 2$ must have high marginal variance in
most of its coordinates. In particular, for any constant $c \in [0,1]$, the top
$c$-fraction of coordinates must have marginal variance at least
$\Omega(d^{1-2/p} \epsilon^2 \frac{1-\delta}{\delta^2})$. For $p > 2$, this
bound grows with the dimension $d$. Our next lemma shows that when $\cD$ is a
product measure of $d$ i.i.d. one-dimension distribution $\cD'$ in the standard
coordinate, the distribution $\cD'$ must be heavy-tailed. The lemma is built upon a fact that $\mathbb{E}\|\eta\|_2\ge \Omega(\epsilon d^{1-1/p}\frac{1-\delta}{\delta})$, with a similar analysis as that of Theorem \ref{thm:lowerBound}. We defer the proof of this fact to the appendix (see Lemma \ref{thm:lowerBound 1st moment}). Note that the fact implies that $\mathbb{E}\|\eta\|_\infty\ge \Omega(\epsilon d^{1/2-1/p}\frac{1-\delta}{\delta})$ by the equivalence between $\ell_2$ and $\ell_\infty$ norms. We then have the following lemma.

\begin{lemma} \label{lemma: suff and necess cond for heavy tails}
  Let $h(\delta)=\frac{1-\delta}{\delta}$ and $p>2$. Let $X_1,...,X_d$ be $d$ random
  variables in $\mathbb R$ sampled i.i.d. from distribution $\mathcal{D}'$. Then
  ``$\mathbb{E}
  \max_{i\in [d]} |X_i|\ge Cd^{1/2-1/p}\epsilon h(\delta)$'' implies
  ``$\Pr_{X\sim \mathcal{D}'}[|X|>x]>\left(\frac{c\epsilon
  h(\delta)}{x}\right)^{2p/(p-2)}$ for some $x> c\epsilon h(\delta)$ with an
  absolute constant $c>0$'', that is, in sufficiently high dimensions,
  $\mathcal{D}'$ is a heavy-tailed
  distribution.
\end{lemma}
\begin{proof}
Denote by $G(x)=\Pr_{X\sim \mathcal{D}'}[|X|>x]$ the complementary Cumulative Distribution Function (CDF) of $\mathcal{D}'$. We only need to show that ``$G(x)\le  \left(\frac{\epsilon h(\delta)}{24x}\right)^{2p/(p-2)}$ for all $x>\frac{\epsilon h(\delta)}{24}$'' implies ``$\mathbb{E} \max_{i\in [d]} |X_i|< Cd^{1/2-1/p}\epsilon h(\delta)$ for a constant $C>0$''. We note that
\begin{equation*}
\begin{split}
\mathbb{E} \max_{i\in[d]} |X_i|&=\int_0^\infty \Pr_{X_i\sim \mathcal{D}}\left[\max_{i\in[d]}|X_i|>x\right]dx\\
&=\int_0^{\frac{\epsilon h(\delta)}{24}} \Pr_{X_i\sim \mathcal{D}}\left[\max_{i\in[d]}|X_i|>x\right]dx+\int_\frac{\epsilon h(\delta)}{24}^\infty [1-(1-G(x))^d]dx\\
&\le \frac{\epsilon h(\delta)}{24}+\frac{\epsilon h(\delta)}{24}\int_1^\infty \left[1-\left(1-\frac{1}{t^{2p/(p-2)}}\right)^d\right]dt\\
&=\frac{\Gamma(\frac{p+2}{2p})\epsilon h(\delta)}{24}\frac{\Gamma(d+1)}{\Gamma(d+\frac{p+2}{2p})}\\
&\sim  d^{1/2-1/p}\epsilon h(\delta),
\end{split}
\end{equation*}
where the second equality holds because for any i.i.d. $Y_i\sim\cD$ with CDF $F(x)$, the CDF of $\max_{i\in[d]} Y_i$ is given by $(1-F(x))^d$, the first inequality holds by the change of variable, and the last $\sim$ relation holds because $\frac{\Gamma(d+1)}{\Gamma(d+\frac{p+2}{2p})}\sim d^{1/2-1/p}$.
\end{proof}

Combining Lemmas \ref{lem:coordinateBounds} and \ref{lemma: suff and necess cond for heavy tails} with Lemma \ref{lem:robustToTV} and the fact that $\mathbb{E}\|\eta\|_\infty\ge \Omega(\epsilon d^{1/2-1/p}\frac{1-\delta}{\delta})$ completes the proof of Theorem \ref{thm:componentProperties}.

\section{Experiments} \label{sec:experiments}

In this section, we evaluate the certified $\ell_\infty$ robustness and verify the tightness of our lower bounds by numerical experiments. Experiments run with two NVIDIA GeForce RTX 2080 Ti GPUs. We release our code and trained models at \url{https://github.com/hongyanz/TRADES-smoothing}.

\subsection{Certified $\ell_\infty$ robustness}

Despite the hardness results of random smoothing on certifying $\ell_\infty$ robustness with large perturbation radius, we evaluate the certified $\ell_\infty$ robust accuracy of random smoothing on the CIFAR-10 dataset when the perturbation radius is as small as 2/255, given that the data dimension $32\times 32\times 3$ is not too high relative to the $2$-pixel attack. The goal of this experiment is to show that random smoothing based methods might be hard to achieve very promising robust accuracy (e.g., $\ge 70\%$) even when the perturbation radius is as small as 2 pixels.

\medskip
\noindent{\textbf{Experimental setups.}}
Our experiments exactly follow the setups of \citep{salman2019provably}. Specifically, we train the models on the CIFAR-10 training set and test it on the CIFAR-10 test sets. We apply the ResNet-110 architecture~\citep{he2016deep} for the CIFAR-10 classification task. The output size of the last layer is 10. Our training procedure is a modification of \citep{salman2019provably}: \citet{salman2019provably} used adversarial training of \cite{madry2017towards} to train a soft-random-smoothing classifier by injecting Gaussian noise. In our training procedure, we replace the adversarial training with TRADES~\citep{zhang2019theoretically}, a state-of-the-art defense model which won the first place in the NeurIPS 2018 Adversarial Vision Challenge~\citep{brendel2020adversarial}. In particular, we minimize the empirical risk of the following loss:
\begin{equation*}
\label{equ: TRADES + random smoothing}
\begin{split}
\min_{f}\ &\mathbb{E}_{X,Y}\mathbb{E}_{\eta\sim\mathcal{N}(0,\sigma^2I)}\Big[\mathcal{L}(f(X+\eta),Y)+\beta\max_{X'\in\mathbb{B}_2(X,\epsilon)} \mathcal{L}(f(X+\eta),f(X'+\eta))\Big],
\end{split}
\end{equation*}
where $\eta$ is the injected Gaussian noise, $\mathcal{L}$ is the cross-entropy loss or KL divergence, $(X,Y)$ is the clean data with label, and $f$ is a neural network classifier which outputs the logits of an instance. For a fixed $f$, the inner maximization problem is solved by PGD iterations, and we update the parameters in the outer minimization and inner maximization problems alternatively.
In our training procedure, we set $\ell_2$ perturbation radius $\epsilon=0.435$, perturbation step size 0.007, number of PGD iterations 10, regularization parameter $\beta=6.0$, initial learning rate 0.1, standard deviation of injected Gaussian noise 0.12, batch size 256, and run 55 epochs on the training dataset. We decay the learning rate by a factor of 0.1 at epoch 50. We use random smoothing of \citet{cohen2019certified} to certify $\ell_2$ robustness of the base classifier. We obtain the $\ell_\infty$ certified radius by scaling the $\ell_2$ robust radius by a factor of $1/\sqrt{d}$. For fairness, we do not compare with the models using extra unlabeled data, ImageNet pretraining, or ensembling tricks.

\begin{table}[t]
	\centering
	\caption{Certified $\ell_\infty$ robustness at a radius of 2/255 on the CIFAR-10 dataset (without extra unlabelled data or pre-trained model).}
	\label{table: certified l_infty}
\begin{tabular}{@{}r|cc@{}}
\toprule
Method & Certified Robust Accuracy & Natural Accuracy\\
\midrule
TRADES + Random Smoothing & 62.6\% & 78.8\%\\
\rowcolor[gray]{.9}
\citet{salman2019provably} & 60.8\% & 82.1\%\\
\citet{zhang2019towards} & 54.0\% & 72.0\%\\
\rowcolor[gray]{.9}
\citet{wong2018scaling} & 53.9\% & 68.3\%\\
\citet{mirman2018differentiable} & 52.2\% & 62.0\%\\
\rowcolor[gray]{.9}
\citet{gowal2018effectiveness} & 50.0\% & 70.2\%\\
\citet{xiao2018training} & 45.9\% & 61.1\%\\
\bottomrule
\end{tabular}
\end{table}

\medskip
\noindent{\textbf{Experimental results.}} We compare TRADES + random smoothing with various baseline methods of certified $\ell_\infty$ robustness with radius $2/255$. We summarize our results in Table \ref{table: certified l_infty}. All results are reported according to the numbers in their original papers.\footnote{We report the performance of \citep{salman2019provably} according to the results: \url{https://github.com/Hadisalman/smoothing-adversarial/blob/master/data/certify/best_models/cifar10/ours/cifar10/DDN_4steps_multiNoiseSamples/4-multitrain/eps_255/cifar10/resnet110/noise_0.12/test/sigma_0.12}, which is the best result in the folder ``best models'' by \citet{salman2019provably}. When a method was not tested under the 2/255 threat model in its original paper, we will not compare with it as well in our experiment.} It shows that TRADES with random smoothing achieves state-of-the-art performance on certifying $\ell_\infty$ robustness at radius 2/255 and enjoys higher robust accuracy than other methods. However, for all approaches, there are still significant gaps between the robust accuracy and the desired accuracy that is acceptable in the security-critical tasks (e.g., robust accuracy $\ge 70\%$), even when the certified radius is chosen as small as 2 pixels.

\subsection{Effectiveness of lower bounds} For random smoothing,
Theorem \ref{thm:componentProperties} suggests that the certified $\ell_\infty$
robust radius $\epsilon$ be (at least) proportional to $\sigma/\sqrt{d}$, where
$\sigma$ is the standard deviation of injected noise. In this section, we verify
this dependency by numerical experiments on the CIFAR-10 dataset and Gaussian
noise.

\medskip
\noindent{\textbf{Experimental setups.}} We apply the ResNet-110
architecture~\citep{he2016deep} for classification.\footnote{The input size of
the architecture is adaptive by applying the adaptive pooling layer.} The output
size of the last layer is 10. We vary the size of the input images with
$32\times 32\times 3$, $48\times 48\times 3$, $64\times 64\times 3$ by calling
the \emph{resize} function. We keep the quantity $\sigma/(\sqrt{d}\epsilon)$ as
an absolute constant by setting the standard deviation $\sigma$ as $0.12$,
$0.18$, and $0.24$, and the $\ell_2$ perturbation radius as $0.435$, $0.6525$,
and $0.87$ in the TRADES training procedure for the three input sizes,
respectively. Our goal is to show that the accuracy curves of the three input
sizes behave similarly. In our training procedure, we set  perturbation step
size 0.007, number of perturbation iterations 10, regularization parameter
$\beta=6.0$, learning rate 0.1, batch size 256, and run 55 epochs on the
training dataset. We use random smoothing~\citep{cohen2019certified} with varying
$\sigma$'s to certify the $\ell_2$ robustness. The $\ell_\infty$ certified
radius is obtained by scaling the $\ell_2$ robust radius by a factor of
$1/\sqrt{d}$.

We summarize our results in Figure \ref{figure: tightness}. We observe that the
three curves of varying input sizes behave similarly. This empirically supports
our theoretical finding in Theorem \ref{thm:componentProperties} that the
certified $\ell_\infty$ robust radius $\epsilon$ should be proportional to the
quantity $\sigma/\sqrt{d}$. In Figure \ref{figure: tightness}, the certified
accuracy is monotonously decreasing until reaching some point where it plummets
to zero. The phenomenon has also been observed by \citet{cohen2019certified} and
was explained by a hard upper limit to the radius we can certify, which is
achieved when all samples are classified by $f$ as the same class.

\begin{figure}
\centering
\includegraphics[scale=0.55]{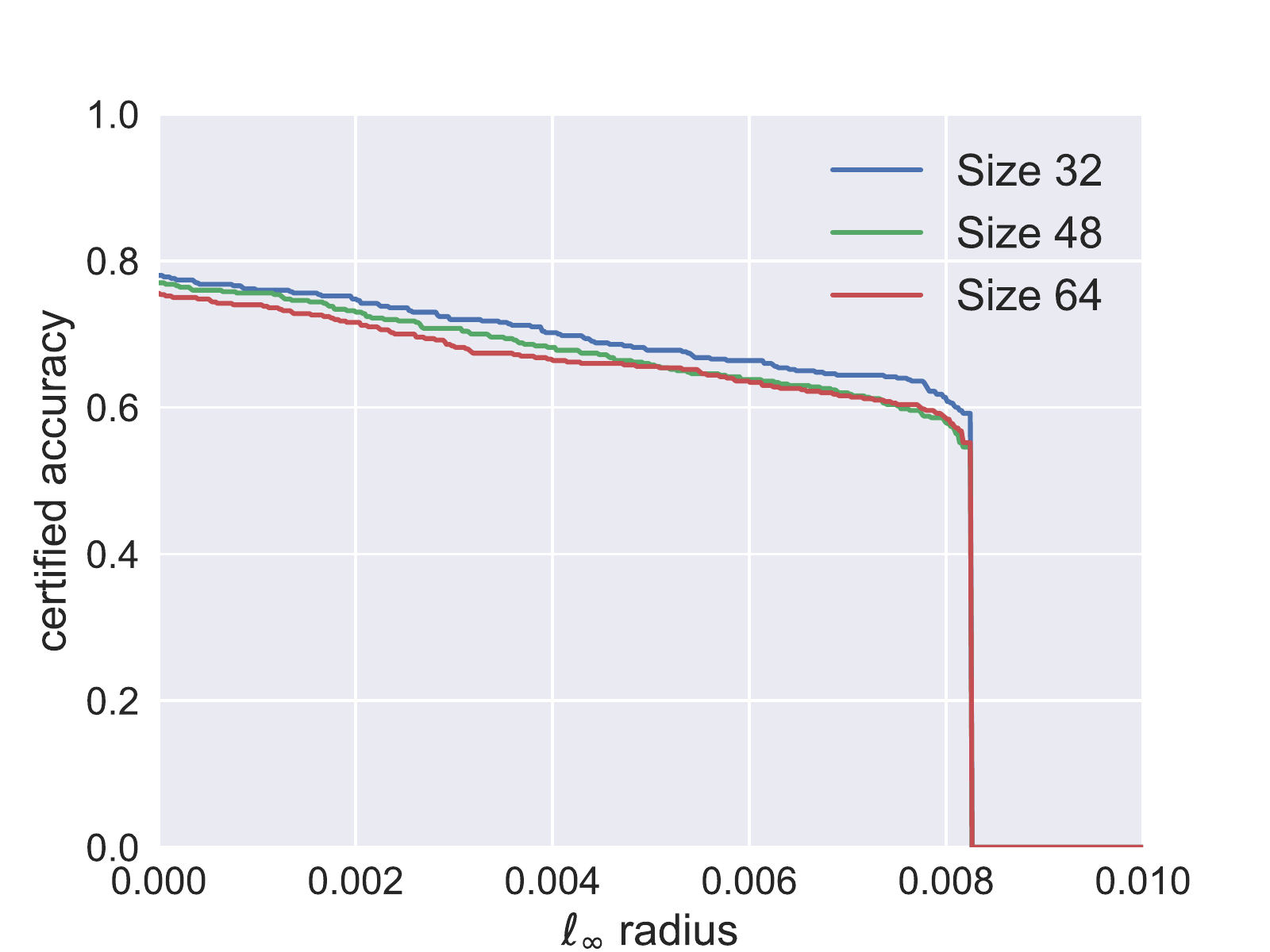}
\caption{Certified accuracy of ResNet-110 models under varying input sizes by random smoothing.}
\label{figure: tightness}
\end{figure}

\section{Conclusions} \label{sec:conclusion}
In this paper, we show a hardness result of random smoothing on certifying adversarial robustness against attacks in the $\ell_p$ ball of radius $\epsilon$ when $p>2$. We focus on a lower bound on the necessary noise magnitude: \emph{any} noise distribution $\cD$ over $\reals^d$ that provides $\ell_p$ robustness with $p>2$ for all base classifiers must satisfy $\expect\eta_i^2=\Omega(d^{1-2/p}\epsilon^2(1-\delta)/\delta^2)$ for 99\% of the features (pixels) of vector $\eta$ drawn from $\cD$, where $\delta$ is the score gap between the highest-scored class and the runner up in the framework of random smoothing. For high-dimensional images where the pixels are bounded in $[0,255]$, the required noise will eventually dominate the useful information in the images, leading to  trivial smoothed classifiers.

The proof roadmap of our results shows that defending against adversarial
attacks in the $\ell_p$ ball of radius $\epsilon$ is almost as hard as defending
against attacks in the $\ell_2$ ball of radius $\epsilon d^{1/2-1/p}$, for random smoothing. We thus
suggest combining random smoothing with dimensionality reduction techniques, such as principal component analysis or auto-encoder, to
circumvent our hardness results, which is left open as future works. Another
related open question is whether one can improve our lower bounds, or show that the bounds are tight.

\paragraph{Acknowledgments.} This work was supported in part by the National Science Foundation under grant CCF-1815011 and by the Defense Advanced Research Projects Agency under cooperative agreement HR00112020003.  The views expressed in this work do not necessarily reflect the position or the policy of the Government and no official endorsement should be inferred. Approved for public release; distribution is unlimited.

\newpage
\appendix

\section{Total-Variation based Robustness} \label{app:TV}
First, we argue that for any points $x$ and $x'$, we must have $g(x) = g(x')$
whenever the total variation distance between $\cD + x$ and $\cD + x'$ is
sufficiently small compared to the gap $\Delta(x)$, where $\cD + x$ denotes the
distribution of $\eta + x$ with $\eta \sim \cD$.
\begin{lemma} \label{lem:directionalRobustness}
  For any distribution $\cD$ on $\reals^d$, base classifier $f : \reals^d \to
  \cY$, and pair of points $x, x' \in \reals^d$, if $\Delta(x) > 2 \tv(\cD + x,
  \cD + x')$, then we have $g(x) = g(x')$.
\end{lemma}
\begin{proof}
  To simplify notation, let $\delta = \tv(\cD + x, \cD + x')$ and let $\eta$ be
  a sample from $\cD$ so that $\eta + x$ is a sample from $\cD + x$ and $\eta +
  x'$ is a sample from $\cD + x'$. By the definition of the total variation
  distance, for any class $y \in \cY$, we have $\delta \geq \bigl|\prob\bigl(f(x
  + \eta) = y\bigr) - \prob\bigl(f(x' + \eta) = y\bigr)\bigr| = |G_y(x) -
  G_y(x')|$. Now let $y = g(x)$ and $y' \neq y$. Then we have
  \begin{equation*}
  \begin{split}
  G_y(x')
  &\geq G_y(x) - \delta\\
  &\geq G_{y'}(x) - \delta + \Delta(x)\\
  &\geq G_{y'}(x') - 2\delta + \Delta(x).
  \end{split}
  \end{equation*}
  Whenever $\Delta(x) > 2\delta$, we are guaranteed that $G_y(x') > G_{y'}(x')$
  for all $y'$, and it follows that $g(x') = y = g(x)$.
\end{proof}

As a consequence of Lemma \ref{lem:directionalRobustness}, we can provide certified
robustness guarantees for the smoothed classifier $g$ in terms of balls defined
by the total variation distance. In particular, for any $x \in \reals^d$ and any
$\delta \in (0,1]$, define
\[
\tvball(x, \delta; \cD) = \{x' \in \reals^d: \tv(\cD + x, \cD + x') < \delta\}
\]
to be the set of points $x'$ around $x$ such that the distributions $\cD + x$
and $\cD + x'$ have total variation distance at most $\delta$. When the
distribution $\cD$ is clear from context, we will simply write $\tvball(x,
\delta)$.

\begin{corollary} \label{cor:tvballRobust}
  For any distribution $\cD$, base classifier $f$, and $x \in \reals^d$ we have
  $g(x') = g(x)$ for all $x' \in \tvball(x, \Delta(x)/2)$.
\end{corollary}

Note that the ball $\tvball(x, \delta)$ is translation invariant (i.e., for any
center $x \in \reals^d$, we have $\tvball(x, \delta) = \tvball(0, \delta) + x$)
and the definition of the ball only depends on the distribution $\cD$.
Therefore, if we can relate the balls for a given distribution $\cD$ to those of
a norm $\norm{\cdot}_p$, then Corollary \ref{cor:tvballRobust} implies robustness with
respect to that norm. Let $\ball{p}(x, r) = \{ x' \in \reals^d: \norm{x' -
x}_p < r \}$ denote the $\ell_p$ ball of radius $r$ centered at $x$.

\begin{corollary}
\label{cor: one direction of tv based robustness}
  Fix any $p > 0$, radius $\epsilon \geq 0$, distribution $\cD$, and let $\delta
  \in [0,1]$ be the smallest total variation bound such that $\ball{p}(0,r)
  \subseteq \tvball(0, \delta)$. For any base classifier $f : \reals^d \to \cY$
  and any point $x \in \reals^d$ with $\Delta(x) > 2\delta$, for all $x' \in
  \ball{p}(x,r)$ we have $g(x') = g(x)$.
\end{corollary}

The following lemma is in an opposite direction as Corollary \ref{cor: one direction of tv based robustness}.

\lemRobustToTV*

\begin{proof}
  Suppose there exists a vector $v \in \cA$ such that $\tv(\cD, \cD+v) >
  \delta$. We show that this implies there is a randomized binary classifier $f
  : \reals^d \to \cY$, such that $g(\cdot; \cD, f)$ is not $(\cA,
  \delta)$-robust. It follows that if $g(\cdot; \cD ,f)$ is $(\cA,
  \delta)$-robust for all randomized classifiers $f$, then we must have
  $\tv(\cD, \cD + v) \leq \delta$ for all $v \in \cA$.

  Fix any $v \in \cA$ such that $\tv(\cD, \cD + v) > \delta$ and let $\cD' = \cD
  + v$ be shorthand notation for the translated distribution. Since $\delta <
  \tv(\cD, \cD') = \sup_{\cS \subseteq \reals^d} \cD(\cS) - \cD'(\cS)$, there
  exists a set $\cS \subseteq \reals^d$ so that $\cD(\cS) - \cD'(\cS) > \delta$.
  Let $\cS^c = \{x \in \reals^d: x \not \in \cS\}$ denote the complement of
  $\cS$. We assume without loss of generality that $\cY = \{0,1\}$ and define
  the randomized classifier $f$ to take value $1$ with probability $\alpha$ and
  $0$ with probability $1-\alpha$ for all $x \in S$ and to take value $1$ with
  probability $\beta$ and $0$ with probability $1-\beta$ for all $x \in
  \mathcal{S}^c$, where $\alpha = 1 - \frac{\cD'(\mathcal{S})}{2}$ and $\beta =
  \frac{1}{2} - \frac{\cD'(\mathcal{S})}{2}$. That is,
  \[
  f(x) = \begin{cases}
    \text{$1$ w.p. $\alpha$ and $0$ otherwise}, & \text{if $x \in \mathcal{S}$}; \\
    \text{$1$ w.p. $\beta$ and $0$ otherwise}, & \text{if $x \in \mathcal{S}^c$}.
  \end{cases}
  \]
  Note that since $\cD'(\mathcal{S}) \in [0,1]$ we have that $\alpha \in [\frac{1}{2}, 1]$
  and $\beta \in [0,\frac{1}{2}]$ are both valid probabilities. For any
  distribution $P$ over $\reals^d$ we have
  \begin{align*}
  \prob_{Z \sim P}(f(Z) = 1)
  &= \prob_{Z \sim P}\bigl(f(Z) = 1 \mid Z \in \mathcal{S}\bigr) \cdot P(\mathcal{S})
    + \prob_{Z \sim P}\bigl(f(Z) = 1 \mid Z \in \mathcal{S}^c\bigr) \cdot P(\mathcal{S}^c) \\
  &= \alpha P(\mathcal{S}) + \beta P(\mathcal{S}^c) \\
  &= P(\mathcal{S}) \cdot (\alpha - \beta) + \beta \\
  &= \frac{1}{2} + \frac{P(\mathcal{S}) - \cD'(\mathcal{S})}{2}.
  \end{align*}
  Therefore, we have $G_1(0; \cD, f) = \prob_{Z \sim \cD}\bigl(f(Z) = 1\bigr) =
  \frac{1}{2} + \frac{\cD(\mathcal{S}) - \cD'(\mathcal{S})}{2} > \frac{1}{2} +
  \frac{\delta}{2}$. Similarly, we have that $G_1(v; \cD, f) = \prob_{Z \sim
  \cD}\bigl(f(Z + v) = 1\bigr) = \prob_{Z \sim \cD'}\bigl(f(Z) = 1\bigr) =
  \frac{1}{2}$. It follows that $g(0;\cD,f) = 1$, $\Delta(0;\cD,f) = 2
  G_1(0;\cD,f) - 1 > 2 (\frac{1}{2} + \frac{\delta}{2}) - 1 = \delta$, and
  $g(v;\cD,f) = 0$ (since $G_1(v;\cD,f) = G_0(v;\cD,f) = \frac{1}{2}$ and the
  ties are broken lexicographically). It follows that $g(\cdot; \cD, f)$ is not
  robust to the adversarial translation $v \in \cA$, as required.
\end{proof}

\section{Total Variation Bounds for Specific Distributions}

\subsection{Isotropic Gaussian}

In this section we give bounds for the total variation distance between shifted
copies of Gaussian distributions with an isotropic covariance matrix. Our
results are derived from the following theorem due to
\citet{Devroye18:GaussianTV}.

\begin{theorem}[Theorem 1.2 of \citet{Devroye18:GaussianTV}]\label{thm:guassianTV}
  Suppose $d > 1$, let $\mu_1 \neq \mu_2 \in \reals^d$ and let $\Sigma_1,
  \Sigma_2$ be positive definite $d \times d$ matrices. Let $v = \mu_1 - \mu_2$
  and let $\Pi$ be a $d \times d-1$ matrix whose columns form a basis for the
  subspace orthogonal to $v$. Define the function
  \begin{equation*}
  tv(\mu_1, \Sigma_1, \mu_2, \Sigma_2)=\max \Bigg\{
  \norm{(\Pi^\top \Sigma_1 \Pi)^{-1} \Pi^\top \Sigma_2 \Pi - I_{d-1}}_F,\frac{|v^\top (\Sigma_1 - \Sigma_2) v|}{v^\top \Sigma_1 v},\frac{v^\top v}{\sqrt{v^\top \Sigma_1 v}}
  \Bigg\},
  \end{equation*}
  where $\norm{\cdot}_F$ denotes the Frobenius norm and $I_{d-1}$ is the
  $(d-1)$-dimensional identity matrix. Then we have
  \[
  \frac{1}{200} \leq \frac{\tv(\normal(\mu_1, \Sigma_1), \normal(\mu_2, \Sigma_2))}{\min\{1, tv(\mu_1, \Sigma_1, \mu_2, \Sigma_2)\}} \leq \frac{9}{2}.
  \]
\end{theorem}

\Cref{thm:guassianTV} takes a simpler form when $\Sigma_1 = \Sigma_2 = \sigma^2
I$ and $\mu_1 = 0$ because then the first and last terms in the max of
$tv(\mu_1, \Sigma_1, \mu_2, \Sigma_2)$ are zero, giving the following:

\begin{corollary}
  Suppose $d > 1$ and let $v \in \reals^d$ and $\sigma > 0$. Then
  \begin{equation*}
  \begin{split}
  \frac{1}{200} \cdot \min\left\{1, \frac{\norm{v}_2}{\sigma} \right\}
  &\leq \tv\bigl(\normal(0, \sigma^2I), \normal(v, \sigma^2I)\bigr)\leq \frac{9}{2} \cdot \min\left\{1, \frac{\norm{v}_2}{\sigma} \right\}.
  \end{split}
  \end{equation*}
\end{corollary}

We can use this result to show that the variance bounds given by
Lemma \ref{lem:coordinateBounds} are nearly tight, except for the dependence on the
total variation bound, $\delta$.

\begin{corollary}
  \label{cor: tightness of Gaussian}
  Fix any $d$, radius $\epsilon > 0$, total variation bound $\delta
  \in [0,1]$, and $p \geq 2$. Setting $\sigma = \frac{9}{2}
  \frac{\epsilon}{\delta} d^{1/2 - 1/p}$ guarantees that for all $v \in
  \reals^d$ with $\norm{v}_p \leq \epsilon$ we have $\tv(\normal(0, \sigma^2 I),
  \normal(v, \sigma^2 I)) \leq \delta$. Moreover, if $\eta \sim \normal(0,
  \sigma^2I)$ then $\expect[\eta_i^2] = \sigma^2 = (\frac{9}{2})^2 \cdot
  \frac{\epsilon^2}{\delta^2} \cdot d^{1-2/p}$ and $\expect[\norm{\eta}_2] \leq
  \frac{9}{2} \frac{\epsilon}{\delta} \cdot d^{1 - 1/p}$.
\end{corollary}
\begin{proof}
  Since for every $v \in \reals^d$ with $\norm{v}_p \leq \epsilon$ we have
  $\norm{v}_2 \leq \epsilon \cdot d^{1/2 - 1/p}$, it is sufficient to choose
  $\sigma$ as in the statement. To bound $\expect[\norm{\eta}_2]$, we use
  Jensen's inequality: $\expect[\norm{\eta}_2] \leq
  \sqrt{\expect[\norm{\eta}_2^2]} = \sqrt{d}\sigma = \frac{9}{2}
  \frac{\epsilon}{\delta} d^{1-1/p}$.
\end{proof}

\subsection{Uniform Distribution on $\ball{\infty}(0,r)$}

In this section, let $\uniform_r$ denote the uniform distribution on
$\ball{\infty}(0,r)$ with density $p_r(x) = \frac{1}{(2r)^d} \ind\{x \in
\ball{\infty}(0,r)\}$.

\begin{lemma}
  For any dimension $d$, any vector $v \in \reals^d$, and any radius $r \geq 0$,
  we have
  \[
  \tv(\uniform_r, \uniform_r + v)
  = 1 - \prod_{i=1}^d \max\left\{0, 1 - \frac{|v_i|}{2r}\right\}.
  \]
\end{lemma}
\begin{proof}
  To simplify notation, let $\cA = \ball{\infty}(0, r)$ and $\mathcal{B} = \ball{\infty}(v,
  r)$. Since $\uniform_r$ has a density function, we can write the total
  variation distance as
  \begin{equation*}
  \begin{split}
  \tv(\uniform_r, \uniform_r + v)
  &= \frac{1}{2} \int_{\reals^d} |p_r(x) - p_r(x - v)| \, dx\\
  &= \frac{1}{2} (2r)^{-d} \int_{\reals^d} |\ind\{x \in \cA\} - \ind\{x \in \mathcal{B}\}| \, dx\\
  &= \frac{1}{2}(2r)^{-d} \vol(\cA \triangle \mathcal{B}),
  \end{split}
  \end{equation*}
  where $\cA \triangle \mathcal{B}$ denotes the symmetric difference of $\cA$ and $\mathcal{B}$. Since
  $\vol(\cA \triangle \mathcal{B}) = \vol(\cA) + \vol(\mathcal{B}) - 2 \vol(\cA \cap \mathcal{B})$, it is sufficient
  to calculate the volume of $\cA \cap \mathcal{B}$. The intersection is a hyper-rectangle
  with side length $\max\{0, 2r - |v_i|\}$ in dimension $i$. Therefore, the
  volume of the intersection is given by $\vol\bigl(\cA \cap \mathcal{B}) = \prod_{i=1}^d
  \max\{0, 2r - |v_i|\}$. Combined with the fact that $\vol(\cA) = \vol(\mathcal{B}) =
  (2r)^d$ this gives
  \begin{equation*}
  \begin{split}
  \tv(\uniform_r, \uniform_r + v)
  &= \frac{1}{(2r)^d} \left((2r)^d - \prod_i \max\{0, 2r - |v_i|\}\right)\\
  &= 1 - \prod_i \max \left\{0, 1 - \frac{|v_i|}{2r}\right\},
  \end{split}
  \end{equation*}
  as required.
\end{proof}

We can also compute the TV-distance for the worst-case shift $v$ with
$\norm{v}_\infty \leq \epsilon$.

\begin{corollary}
  For any $\epsilon \geq 0$, the vector $v = (\epsilon, \dots, \epsilon) \in
  \reals^d$ satisfies
  $$v \in \argmax_{v : \norm{v}_\infty \leq \epsilon}
  \tv(\uniform_r, \uniform_r + v),$$
  and $\tv(\uniform_r, \uniform_r + v) =
  \min\{1, 1 - (1 - \frac{\epsilon}{2r})^d\}$. Finally, for $\epsilon \in
  [0,r]$, we have
  \[
  1 - e^{-\frac{d \epsilon}{2r}}
  \leq \max_{v : \norm{v}_\infty \leq \epsilon} \tv(\uniform_r, \uniform_r + v)
  \leq 1 - 4^{-\frac{d \epsilon}{2r}}.
  \]
\end{corollary}
\begin{proof}
  To see that $v = (\epsilon, \dots, \epsilon)$ is a maximizer, observe that the
  optimization problem decouples over the components $v_i$ and that to maximize
  the term corresponding to component $v_i$ we want to choose $|v_i|$ as large
  as possible. It follows that all vectors $v \in \{\pm \epsilon\}^d$ are
  maximizers.

  The bounds for when $\epsilon \in [0,r]$ follow from the fact that for any $z
  \in [0,\frac{1}{2}]$, we have $4^{-x} \leq 1 - z \leq e^{-z}$ applied with $z
  = 1 - \frac{\epsilon}{2r}$.
\end{proof}

\begin{corollary}
  Fix any dimension $d$, radius $\epsilon > 0$, and total variation bound
  $\delta \in [0,1]$. Setting $r = \frac{1}{2} \frac{\epsilon}{\delta} d
  \log(4)$ guarantees that for all $v \in \reals^d$ such that $\norm{v}_\infty
  \leq \epsilon$ we have $\tv(\uniform_r, \uniform_r + v) \leq \delta$.
  Moreover, if $\eta \sim \uniform_r$ then $\expect[\eta_i^2]$ is the variance
  of a uniform random variable on $[-r, r]$, which is
  $\frac{\sqrt{2}\log(4)^2}{48} \frac{\epsilon^2}{\delta^2} d^2 \leq
  \frac{\epsilon^2}{\delta^2} d^2$.
\end{corollary}
\begin{proof}
  This follows by determining the smallest value of $r$ for which $1 -
  4^{-\frac{d\epsilon}{2\delta}} \leq \delta$.
\end{proof}

\section{Lower Bound on Noise $\ell_2$-Norm}

\begin{lemma}\label{thm:lowerBound 1st moment}
  Fix any $p \geq 2$ and let $\cD$ be a distribution on $\reals^d$ such that
  there exists a radius $\epsilon$ and total variation bound $\delta$ satisfying
  that for all $v \in \reals^d$ with $\norm{v}_p \leq \epsilon$ we have
  $\tv(\cD, \cD+v) \leq \delta$. Then
  \[
  \expect_{\eta \sim \cD} \norm{\eta}_2
  \geq \frac{\epsilon d^{1 - 1/p}}{24} \cdot \frac{1 - \delta}{\delta}.
  \]
\end{lemma}

\begin{proof}
We first prove the lemma in the one-dimensional case.
Let $\eta' = \eta + \epsilon\in\reals$ so that $\eta'$ is a sample from $\cD +
  \epsilon$ and define $r = \epsilon / 2$ so that the sets $\mathcal{A} =
  (-r,r)$ and $\mathcal{B} = (\epsilon - r, \epsilon + r)$ are disjoint. From
  Markov's inequality, we have that $\prob(\eta \in \mathcal{A}) = 1 -
  \prob(|\eta| \geq r) \geq 1 - \frac{\expect |\eta|}{r}$. Further, since $\eta'
  \in \mathcal{B}$ if and only if $\eta \in \mathcal{A}$, we have $\prob(\eta'
  \in \mathcal{B}) \geq 1 - \frac{\expect |\eta|}{r}$. Next, since $\mathcal{A}$
  and $\mathcal{B}$ are disjoint, it follows that $\prob(\eta' \in \mathcal{A})
  \leq 1 - \prob(\eta' \in \mathcal{B}) \leq 1 - 1 + \frac{\expect |\eta|}{r} =
  \frac{\expect |\eta|}{r}$. Finally, we have $\delta \geq \prob(\eta \in
  \mathcal{A}) - \prob(\eta' \in \mathcal{A}) \geq 1 - \frac{2 \expect
  |\eta|}{r} = 1 - \frac{4 \expect |\eta|}{\epsilon}$. Rearranging this
  inequality proves the claim that $\mathbb{E}|\eta|\ge (1-\delta)\epsilon/4$. Combining with Lemma \ref{lem: one dimensional lower bound small delta}, we obtain that $\expect |\eta|
  \geq \frac{\epsilon}{12} \cdot \frac{1-\delta}{\delta}$, due to the fact that for any $\delta \in (0,1]$ we have
$\max\{\frac{1-\delta}{4}, \frac{(1-\delta)^2}{8\delta}\} \geq \frac{1}{12}
\cdot \frac{1-\delta}{\delta}$.

We now prove the $d$-dimensional case. Let $\eta$ be a sample from $\cD$. By scaling the vectors from
  Corollary \ref{cor:badDirections} by $\epsilon$, we obtain $b > d/2$ vectors $v_1,
  \dots, v_b \in \reals^d$ with $\norm{v_i}_p = \epsilon$ and $\norm{v_i}_2 =
  \epsilon \cdot b^{1/2 - 1/p}$. By assumption we must have $\tv(\cD, \cD + v_i)
  \leq \delta$, since $\norm{v_i}_p \leq \epsilon$, and the above-mentioned one-dimensional case
  implies that $\expect \frac{|v_i^\top \eta|}{\norm{v_i}_2} \geq
  \frac{\norm{v_i}_2}{12} \frac{1-\delta}{\delta}$ for each $i$. We use
  this fact to bound $\expect \norm{\eta}_2$.

  Let $\bQ \in \reals^{b \times d}$ be the matrix whose $i^\text{th}$ row is
  given by $v_i / \norm{v_i}_2$ so that $\bQ$ is the orthogonal projection
  matrix onto the subspace spanned by the vectors $v_1, \dots, v_b$. Then we
  have
  $$ \expect \norm{\eta}_2 \geq \expect \norm{\bQ \eta}_2\geq \frac{1}{\sqrt{b}}
  \expect \norm{\bQ \eta}_1= \frac{1}{\sqrt{b}} \sum_{i=1}^b \expect
  \frac{|v_i^\top \eta|}{\norm{v_i}_2}\geq \frac{1}{\sqrt{b}} \sum_{i=1}^b
  \frac{\norm{v_i}_2}{12} \frac{1-\delta}{\delta},$$
  where the first inequality follows because orthogonal projections are
  non-expansive, the second inequality follows from the equivalence of $\ell_2$
  and $\ell_1$ norms, and the last inequality follows from $\expect \frac{|v_i^\top \eta|}{\norm{v_i}_2} \geq
  \frac{\norm{v_i}_2}{12} \frac{1-\delta}{\delta}$.
  Using the fact that $\norm{v_i}_2 = \epsilon \cdot b^{1/2 - 1/p}$, we have
  that $\expect \norm{\eta}_2 \geq \frac{\epsilon b^{1 - 1/p}}{12} \cdot \frac{1
  - \delta}{\delta}$. Finally, since $b > d/2$ and $(1/2)^{1 - 1/p} \geq 1/2$
  for $p \geq 2$, we have $\expect \norm{\eta}_2 \geq \frac{\epsilon d^{1 -
  1/p}}{24} \cdot \frac{1 - \delta}{\delta}$, as required.
\end{proof}

\bibliography{references}

\end{document}